\documentclass{article}

\usepackage[final]{bdl_2021_camera_ready}

\usepackage[utf8]{inputenc} 
\usepackage[T1]{fontenc}    
\usepackage{hyperref}       
\usepackage{url}            
\usepackage{booktabs}       
\usepackage{amsfonts}       
\usepackage{nicefrac}       
\usepackage{microtype}      
\usepackage{xcolor}         
\usepackage{mathtools}
\usepackage{amsthm}
\usepackage{epstopdf}
\usepackage{tikz}
\usepackage[export]{adjustbox}
\usetikzlibrary{calc} 
\usepackage{graphics}

\hypersetup{
	colorlinks=true,
	urlcolor=blue,
	citecolor=blue,
	linktoc=all,
	linkcolor=blue}

\newtheorem{theorem}{Theorem}[section]

\newtheorem{lemma}{Lemma}[section]

\newtheorem{corollary}{Corollary}[section]

\newtheorem{remark}{Remark}[section]

\newtheorem{manualtheoreminner}{Theorem}
\newenvironment{manualtheorem}[1]{%
  \renewcommand\themanualtheoreminner{#1}%
  \manualtheoreminner
}{\endmanualtheoreminner}

\newtheorem{manualcorollaryinner}{Corollary}
\newenvironment{manualcorollary}[1]{%
  \renewcommand\themanualcorollaryinner{#1}%
  \manualcorollaryinner
}{\endmanualcorollaryinner}

\newcommand{\ddr}{\mathrm{d}} 
\newcommand{\bx}{\boldsymbol{x}}
\newcommand{\bX}{\boldsymbol{X}}

\newcommand{\bw}{\boldsymbol{w}}
\newcommand{\bW}{\boldsymbol{W}}
\newcommand{\bv}{\boldsymbol{v}}
\newcommand{\be}{\boldsymbol{e}}

\newcommand{\bg}{\boldsymbol{g}}
\newcommand{\bh}{\boldsymbol{h}}
\newcommand{\T}{{\text{\tiny\sffamily\upshape\mdseries T}}}

\def\theoremDependence{Consider a Bayesian neural network as described in Equation~\eqref{eq:hidden_unit_form} with some activation function $\phi$. Let elements of weight vector $\bW^{(\ell)}$ follow some zero-center elliptical (possibly different and possibly dependent) distributions, and weight vectors be independent for distinct units of the following layer. 
If~$\ell = 1$, then $\Delta^{(\ell)}(z_1, z_2) = 0$ for all $z_1, z_2$. 
If $\ell \ge 2$, then $\Delta^{(\ell)}(z_1, z_2) \ge 0$ if $z_1 z_2 \ge 0$, and $\Delta^{(\ell)}(z_1, z_2) \le 0$ otherwise. If (and only if) the activation function $\phi$ satisfies $\mathbb{P} \left( \phi(\bg^{(\ell)}) = 0 \right) = 0$ for any $\bg^{(\ell)}$, then $\Delta^{(\ell)}(0, z_2) = 0$ and $\Delta^{(\ell)}(z_1, 0) = 0$ for any $z_1, z_2$. 
}

\def\corollaryTauRho{In Bayesian neural networks under assumptions of Theorem~\ref{theorem:hidden_units_dependence}, Kendall's tau and Spearman's rho computed for hidden units are equal to zero.}

\def\theoremCovariance{Consider a Bayesian neural network as described in Equation~\eqref{eq:hidden_unit_form} with ReLU activation function. Assume that weights $w^{(\ell)}$ are centered and independent from units $h^{(\ell - 1)}$.
If weights are~uncorrelated, then any pre-activations of the same layer~$\ell$ are~uncorrelated.}

\def\lemmaHiddenUnitsDependence{Let $X_1, \dots, X_N$ be some possibly dependent random variables and $W_1, \dots, W_N$ be symmetric, mutually independent and independent from $X_1, \dots, X_N$, then random variables $X_1 W_1, \dots, X_N W_N$ satisfy the PD condition.
}

\title{Dependence between Bayesian neural network units}

\author{Mariia Vladimirova\footnote[1]{} 
\qquad
  Julyan Arbel 
\qquad
  St\'ephane Girard \\
Univ. Grenoble Alpes, Inria, CNRS, Grenoble INP, LJK, 38000 Grenoble, France
 }

\begin{document}

\maketitle
\begin{abstract}
The connection between Bayesian neural networks and Gaussian processes gained a~lot of attention in the last few years, with the flagship result that hidden units converge to a Gaussian process limit when the layers width tends to infinity. Underpinning this result is the fact that hidden units become independent in the infinite-width limit. 
Our aim is to shed some light on hidden units dependence properties in practical  finite-width Bayesian neural networks. In addition to theoretical results, we assess empirically the depth and width impacts on hidden units dependence properties.
\end{abstract}
\footnotetext[1]{Corresponding author: \href{mailto:mariia.vladimirova@inria.fr}{mariia.vladimirova@inria.fr}.}

\section{Introduction}

Pre-activations and post-activations of layer $\ell$ in Bayesian neural networks are respectively defined as
\begin{equation}
\label{eq:hidden_unit_form}
    \bg^{(\ell)} = \bW^{(\ell)^\T} \bh^{(\ell - 1)}, \quad  \bh^{(\ell)} = \phi(\bg^{(\ell)}),
\end{equation}
where $\bW^{(\ell)} \in \mathbb{R}^{H_{\ell-1}}\times\mathbb{R}^{H_\ell}$ are weights that follow some prior distribution, $\phi$ is a nonlinear function called activation function, $\bg^{(\ell)} \in \mathbb{R}^{H_\ell}$ is a vector of pre-activations, and $\bh^{(\ell)} \in \mathbb{R}^{H_\ell}$ is a vector of post-activations. For $\ell=0$, $\bh^{(0)}$ is an input vector of deterministic numerical object features. For $\ell > 0$, $H_{\ell}$ is the width of layer $\ell$. When we talk about both $\bg^{(\ell)}$ or $\bh^{(\ell)}$ or when we do not need to specify if we consider pre-activations or post-activations, we refer to units of layer $\ell$. The distributions induced on units are priors in functional space, or induced priors, also called prior predictives in the literature. 

Induced priors in Bayesian neural networks with Gaussian weights become Gaussian processes when the number of hidden units per layer tends to infinity~\citep{neal1996bayesian,matthews2018gaussian,lee2018deep,garriga2019deep}. Stable distributions also lead to stable processes which are generalizations of Gaussian ones~\citep{favaro2020stable}.
Tightening hidden units closer to the Gaussian process can be considered as reducing the induced dependence between units. 
Since it is not the case for finite-width neural networks, dealing with the induced dependence is one of the~problems in describing the prior predictive. 

In this note, we focus on dependence properties that help in better characterizing hidden unit priors. 
We study dependence properties between hidden units in Bayesian neural networks and establish analytically, in Section~\ref{section:dependence_properties}, and empirically, as illustrated on Figure~\ref{fig:dependence_influence_width-depth}, positive and negative dependence induced by weight priors. 
\begin{figure}[ht!]
\centering
\vspace{-0.5cm}
\begin{tikzpicture}
    \node[inner sep=0pt] (d2w2) at (0,0)
    {\includegraphics[height=0.3\linewidth, trim={0.5cm 0 0 0}, clip]{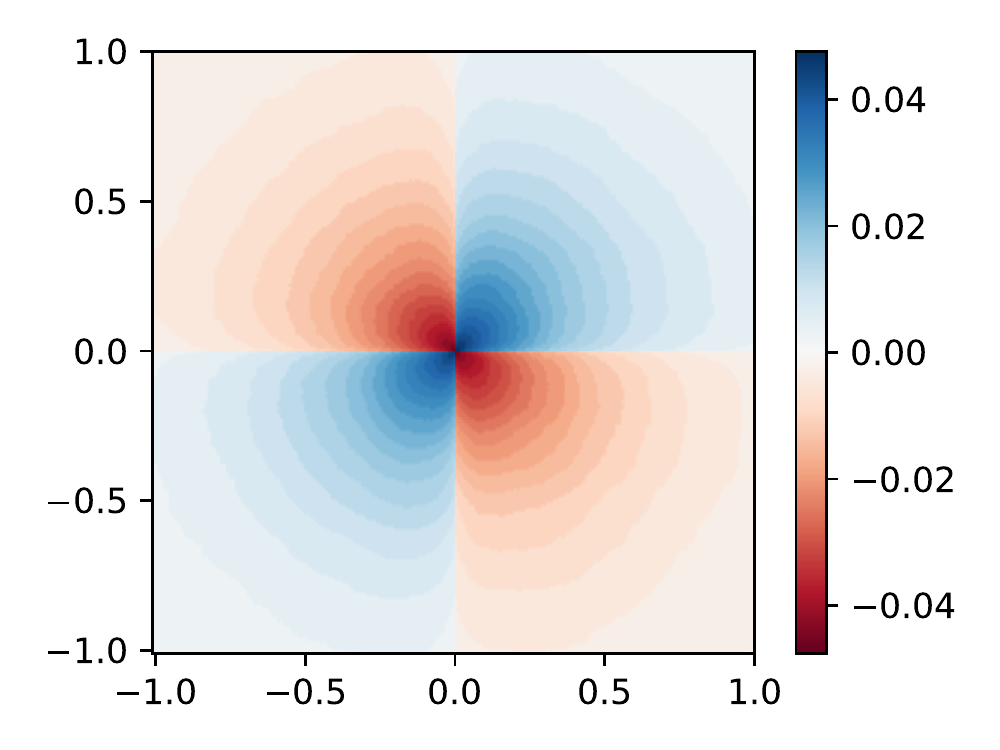}};
    \node [above, xshift=-0.3cm, yshift=-0.4cm] (d2w2_title) at (d2w2.north) {{\small $L=2, H=2$}};
    \node [above right] (d2w5) at (d2w2.east)
    {\includegraphics[height=0.3\linewidth, trim={0.9cm 0 1.7cm 0}, clip]{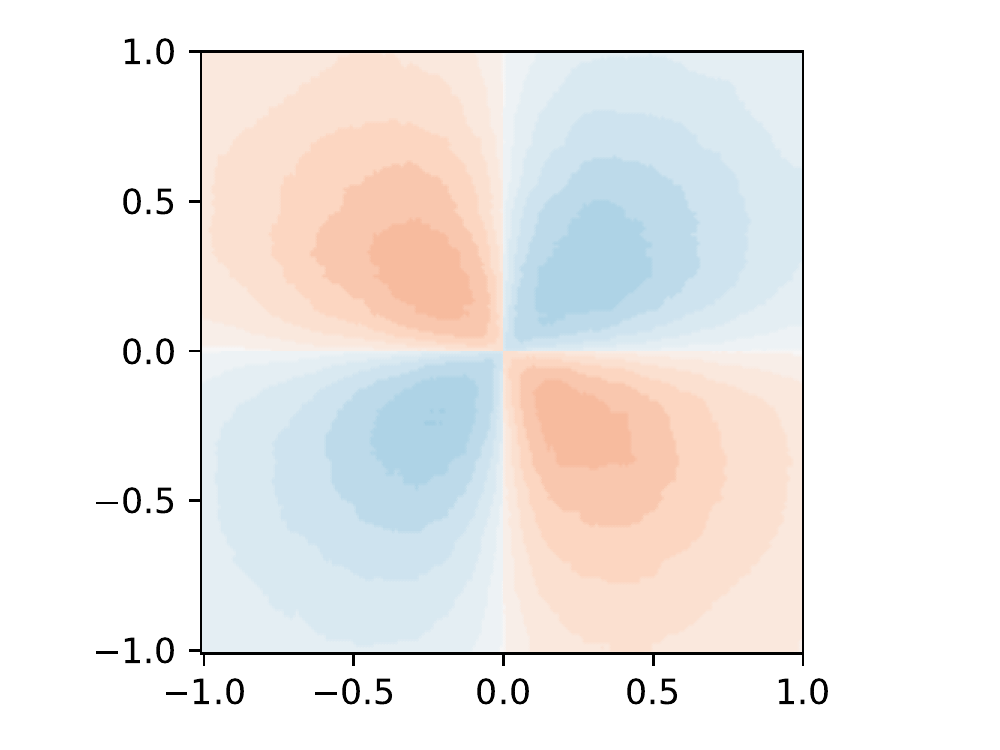}};
    \node [above, yshift=-0.5cm, xshift=0.2cm] (d2w5_title) at (d2w5.north) {{\small $L=2, H=5$}};
    
    \node [right] (d2w10) at (d2w5.east)
    {\includegraphics[height=0.3\linewidth, trim={0.9cm 0 1.7cm 0}, clip]{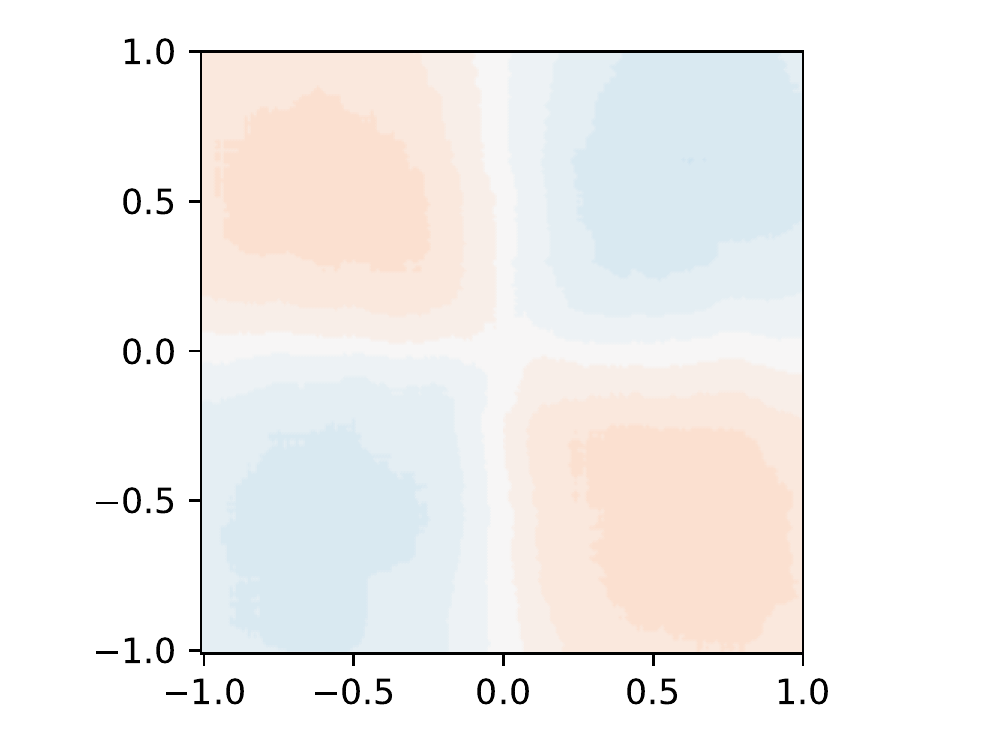}};
    \node [above, yshift=-0.5cm, xshift=0.2cm] (d2w10_title) at (d2w10.north) {{\small $L=2, H=10$}};
     
    \draw[->, in=180, out=50] (d2w2.north)++(0, 0.25)  to node[above] {Increasing $H$} ($(d2w5.west)+(0,0.7)$) ;
    \node [below right] (d3w2) at (d2w2.east)
    {\includegraphics[height=0.3\linewidth, trim={0.9cm 0 1.7cm 0}, clip]{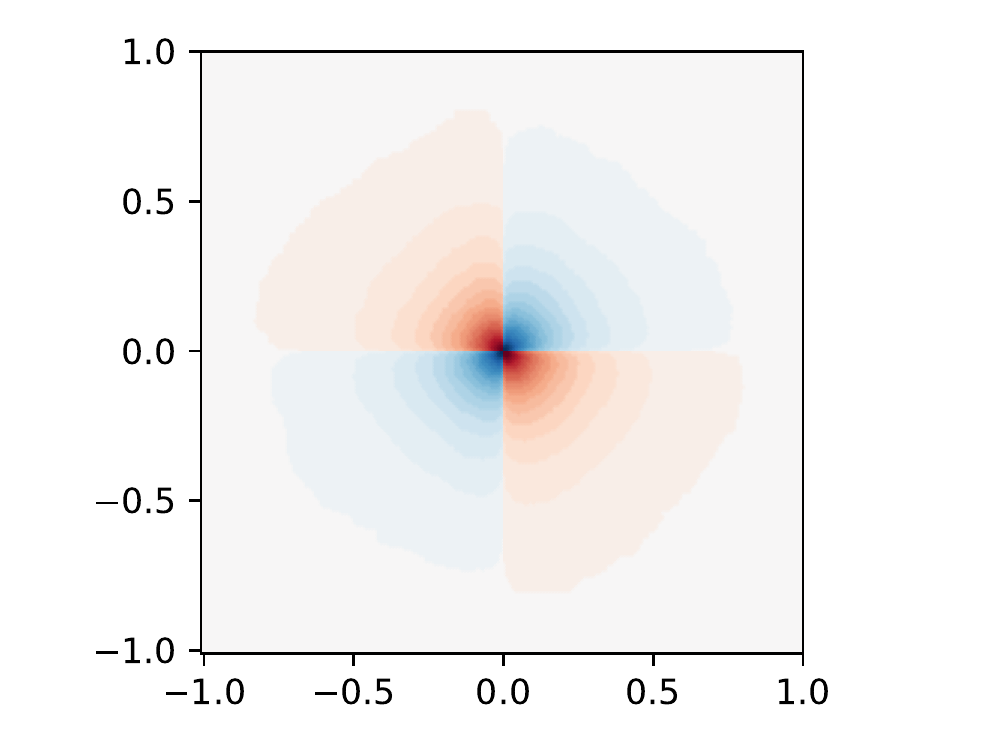}};
    \node [above, yshift=-0.5cm, xshift=0.2cm] (d3w2_title) at (d3w2.north) {{\small $L=3, H=2$}};
    
    \node [right] (d4w2) at (d3w2.east)
    {\includegraphics[height=0.3\linewidth, trim={0.9cm 0 1.7cm 0}, clip]{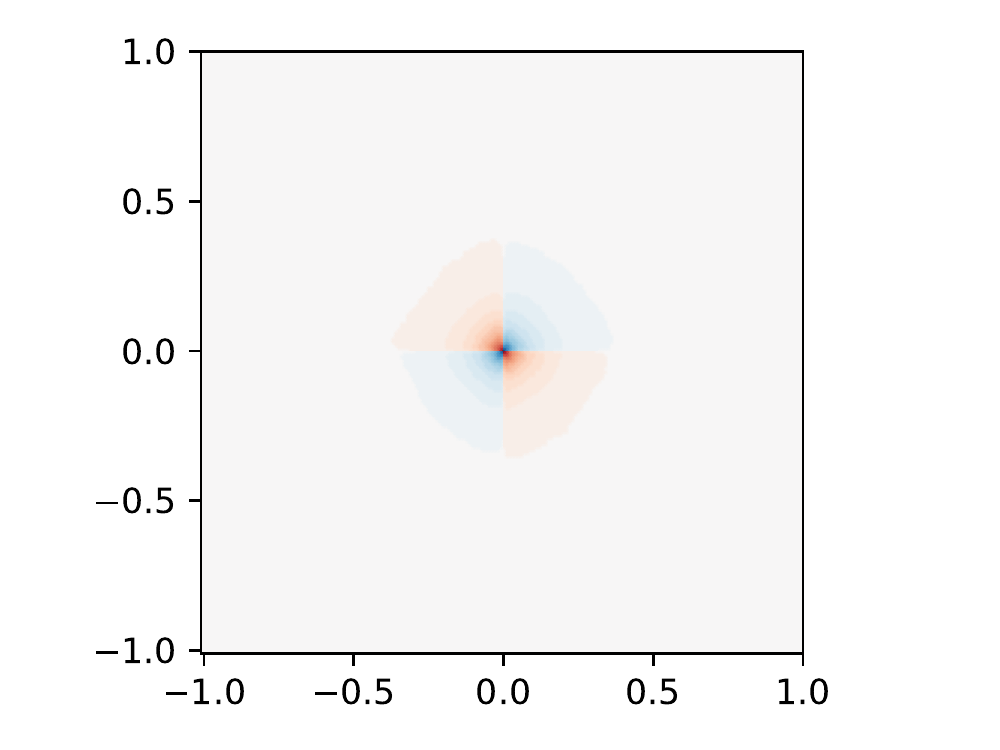}};
    \node [above, yshift=-0.5cm, xshift=0.2cm] (d4w2_title) at (d4w2.north) {{\small $L=4, H=2$}};
    \draw[->, in=180, out=-50] (d2w2.south)++(0, -0.25) to node[below] {Increasing $L$} ( $ (d3w2.west)+(0,-0.7) $ );
\end{tikzpicture}
\caption{Influence of neural network width $H$ (first row) and depth $L$ (second row) on the units dependence measured through $\Delta^{(L)}(z_1, z_2)$, defined in Equation~\eqref{eq:def-delta}.
}
\label{fig:dependence_influence_width-depth}
\end{figure}
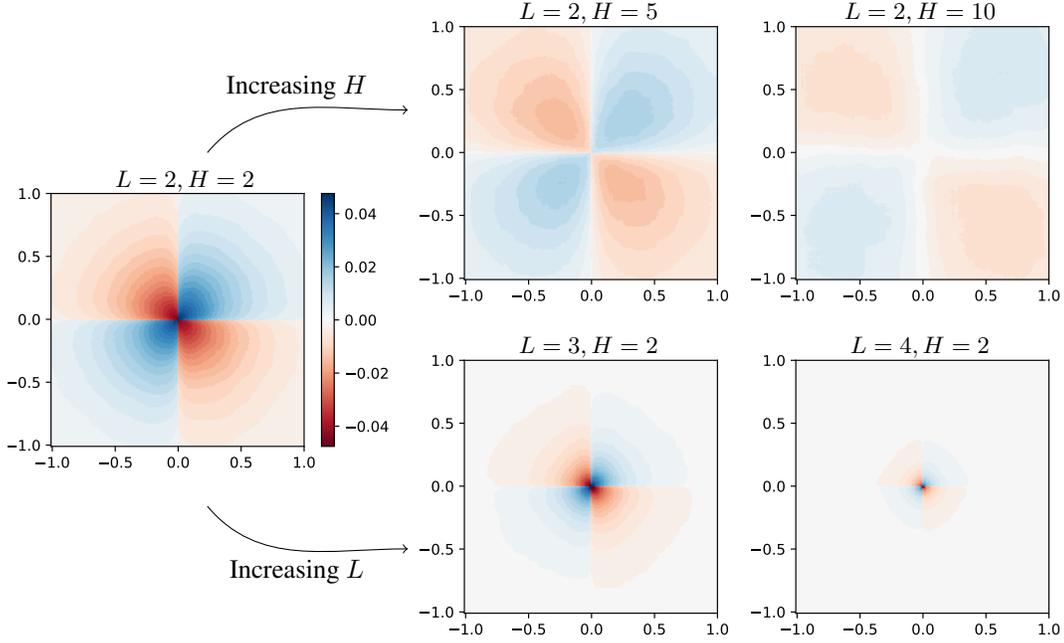

\section{Dependence properties}
\label{section:dependence_properties}

We start by showing that 
hidden units of the same layer are uncorrelated for uncorrelated weights. 
This theorem refines the non-negative covariance theorem  from~\citet{vladimirova2018bayesian}, the proof is deferred to Appendix~\ref{appendix:additional_lemmas}.

\begin{theorem}[Covariance between hidden units]
\label{theorem:non-negative_covariance}
\theoremCovariance 
\end{theorem}

Let $g^{(\ell)}$ and $\tilde g^{(\ell)}$  be two distinct pre-nonlinearities of layer $\ell$, and define
\begin{equation}\label{eq:def-delta}
    \Delta^{(\ell)}(z_1, z_2) := \mathbb{P}(g^{(\ell)} \geq z_1, \tilde g^{(\ell)} \geq z_2)-\mathbb{P}(g^{(\ell)}\geq z_1)\mathbb{P}(\tilde g^{(\ell)}\geq z_2).
\end{equation}
The following theorem represents how the sign of $\Delta^{(\ell)}(z_1, z_2)$ depends on signs of $z_1$ and $z_2$. 
Usually the weights in Bayesian neural networks are assumed to be independent~\citep{neal1996bayesian,matthews2018gaussian,lee2018deep,garriga2019deep}. 
However, some works \citep{garriga2021correlated, fortuin2021bayesian} proposed correlated priors for convolutional neural networks since trained weights are empirically strongly correlated. They showed that these correlated priors can improve overall performance. 
Our results take into account Bayesian neural networks with possibly dependent priors. More precisely, for $\ell$-th layer pre-activations $g_j = \sum_{i=1}^{H_{\ell-1}} W_{ij} h_i$, $j\in \{1, \dots, H_\ell\}$, weights $W_{i_1j}$ and $W_{i_2j}$ can be dependent for distinct $i_1, i_2 \in \{1, \dots, H_{\ell-1}\}$, while $W_{i j_1}$ and $W_{i j_2}$ are independent for any distinct $j_1, j_2 \in \{1, \dots, H_\ell\}$. By applying Lemma~\ref{lemma:delta_for_z_positive} and \ref{lemma:delta_for_z_equal_zero} from Appendix, we have the relationship between $\Delta$ and values of $z_1, z_2$:
\begin{theorem}[Hidden units dependence]
\label{theorem:hidden_units_dependence}
\theoremDependence
\end{theorem}
The case of $\mathbb{P} \left( \phi(\bg^{(\ell)}) = 0 \right) = 0$ corresponds to post-activations without critical mass at zero. They can be obtained after applying activation functions such as identity, sigmoid, ELU, and others, but not ReLU.

\begin{remark}
\label{remark:delta_with_upper_bounds}
Due to the ellipticity and zero-centering of distributions, the statement of Theorem~\ref{theorem:hidden_units_dependence} is also true for $\tilde \Delta^{(\ell)}(z_1, z_2) := \mathbb{P}(g^{(\ell)} \leq z_1, \tilde g^{(\ell)} \leq z_2)-\mathbb{P}(g^{(\ell)}\leq z_1)\mathbb{P}(\tilde g^{(\ell)}\leq z_2)$.
\end{remark}

\subsection{Corollaries}

Dependence measures and properties are interrelated. 
Widely used measures such as Kendall's tau and Spearman's rho \citep{nelsen2007introduction}, take into account the concordance. Based on Theorem~\ref{theorem:hidden_units_dependence}, we establish that for hidden units these coefficients are equal to zero.
\begin{corollary}
\label{corollary:tau_rho}
\corollaryTauRho
\end{corollary}

The following dependence condition of hidden units is defined by~\citet{vladimirova2021bayesian} in order to establish the Weibull-tail property of hidden units. 
Random variables $X_1, \dots, X_N$ satisfy the \textit{positive dependence (PD) condition} if the following inequalities hold for all $z \in \mathbb{R}$ and some constant~$C> 0$:
\begin{align*}\mathbb{P}\left(X_1\ge 0,\ldots,X_{N-1} \ge 0 | X_N \ge z\right) &\ge C \quad \text{(right tail),}\\ 
\mathbb{P}\left(X_1\le 0,\ldots,X_{N-1} \le 0 | X_N \le z \right) &\ge C \quad \text{ (left tail).}
\end{align*}
The proof of Theorem~\ref{theorem:hidden_units_dependence} can be adapted to prove the following property for hidden units, originally proved in \citet{vladimirova2021bayesian}.
\begin{corollary}[\citealp{vladimirova2021bayesian}]
\label{cor:units_dependence_condition}
\lemmaHiddenUnitsDependence
\end{corollary}


\section{Experiments}
\label{section:experiments}

We have built neural networks of $L = 2, 3, 4$ hidden layers, with $H = 2, 5, 10$ hidden units on each layer. 
We used a fixed input $\mathbf{x}$ of size $10^4$, which can be thought of as an image of dimension $100\times 100$. This input was sampled once for all with standard Gaussian entries. 
In order to obtain samples from the prior distribution of the neural network units, we have sampled the weights from independent centered Gaussians from which units were obtained by forward evaluation with the ReLU non-linearity. This process was iterated $n =10^5$ times. We propagated the priors and calculated values of $\Delta^{(L)}$, defined in Equation~\eqref{eq:def-delta}, for $z_1, z_2$ on a grid $(-1.0, 1.0)\times (-1.0, 1.0)$. The~results are illustrated on Figure~\ref{fig:dependence_influence_width-depth}. All subplots are appeared to be divided into four quadrants of negative and positive values, confirming Theorem~\ref{theorem:hidden_units_dependence}: $\Delta^{(L)}$ is positive when $z_1$ and $z_2$ are of the same sign, and $\Delta^{(L)}$ is negative otherwise.

The increase of the number of hidden units $H$ leads to less dependence between hidden units as~the~obtained values~of~$\Delta^{(L)}$ are smaller. Moreover, the center of the plot takes values closer to zero than the corners. The $\Delta^{(L)}$ values are more spread out and less peaked. 
The increase of the depth $L$ leads to the opposite result when the corners become closer to zero than the center while the $\Delta^{(L)}$ values become more peaked around zero. 

\section{Discussion}

We described analytically and empirically the dependence between hidden units in Bayesian neural networks. We proved that Kendall's tau and Spearman's rho are equal to zero. These results help to understand better the influence of changing the width and depth in Bayesian neural networks. 

\paragraph{Representation learning.} \citet{aitchison2020bigger} studied the prior over representations in finite and infinite Bayesian neural networks. The narrower, deeper networks offer more flexibility because the covariance of the outputs gradually disappears as network size increases. The results are obtained by considering the variability in the top-layer kernel induced by the prior over a finite neural network. Our empirical results show that such deep narrow neural networks keep hidden units highly dependent in the center. Therefore, there might be a connection between the prior over representations and highly-peaked dependence between units. 

\paragraph{Width-depth trade-off.} From a deep Gaussian process perspective, \citet{pleiss2021limitations} argue that width becomes harmful to model fit and performance as the posterior becomes less data-dependent with width. Empirically, there is a sweet spot in width for convolutional neural networks, depending on the dataset. The increase of width beyond this sweet spot degrades the performance. The tail analysis demonstrates that width and depth have opposite effects: depth accentuates a model's non-Gaussianity, while width makes models increasingly Gaussian. Indeed, it was proved that Bayesian neural network units are heavier-tailed with depth~\citep{vladimirova2018bayesian,zavatone2021exact,noci2021precise,vladimirova2021bayesian}. So the increase of width might make the resulting units distributions more Gaussian in the center.



\bibliographystyle{apalike}

\appendix

\section{Bayesian neural network properties}
\label{appendix:additional_lemmas}

\subsection{Covariance}

Further, we provide the proof of the following theorem that refines the non-negative theorem from~\citet{vladimirova2018bayesian}. 
\begin{manualtheorem}{\ref{theorem:non-negative_covariance}}
\theoremCovariance 
\end{manualtheorem}
\begin{proof}
 
Consider first hidden layer distinct pre-activations $g^{(1)} = \bW^{(1)^\T} \bh^{(0)}$ and $\tilde g^{(1)} = \tilde \bW^{(1)^\T} \bh^{(0)}$ as described in Equation~\eqref{eq:hidden_unit_form}. Since $h^{(0)}$ is a deterministic vector, the covariance between pre-activations is of the same sign as the covariance between the weights: 
\begin{align*}
    \text{Cov} \left[ \bW^{(1)^\T} \bh^{(0)}, \tilde \bW^{(1)^\T} \bh^{(0)} \right] 
    = \sum_{i = 1}^{H_{1}}  \sum_{j = 1}^{H_{1}} \left(\mathbb{E}\left[ W_i^{(1)} \tilde W_j^{(1)} \right]  - \mathbb{E}\left[ W_i^{(1)} \right]  \mathbb{E}\left[ \tilde W_j^{(1)}  \right] \right) h_i^{(0)} h_j^{(0)}.
\end{align*}
If the weights are uncorrelated, then the units are uncorrelated, therefore, $ \text{Cov} \left[ g^{(1)}, \tilde g^{(1)}\right] = 0$. 

  
Consider the case where $\ell \ge 2$.
Let $\bX \in \mathbb{R}^{H_{\ell-1}}$ be outputs of hidden layer $\ell - 1$, $\bW \in \mathbb{R}^{H_{\ell-1}}$ be weights that follow some prior distribution, $\tilde \bW \in \mathbb{R}^{H_{\ell-1}}$ be an independent copy of $\bW$. Two distinct units of layer $\ell$ can be written as $g^{(\ell)} = \bW^\T \bX$ and $\tilde g^{(\ell)} = \tilde \bW^\T \bX$. Then, the covariance between pre-activations $g^{(\ell)}$ and $\tilde g^{(\ell)}$  can be expressed as
\begin{equation*}
    \text{Cov} \left[ g^{(\ell)}, \tilde g^{(\ell)} \right] 
    = \sum_{i = 1}^{H_{\ell - 1}}  \sum_{j = 1}^{H_{\ell - 1}} \left(\mathbb{E}[ W_i \tilde W_j ] \mathbb{E}\left[ X_i X_j\right] - \mathbb{E}\left[ W_i \right]  \mathbb{E}[ \tilde W_j ] \mathbb{E}\left[ X_i\right]  \mathbb{E}\left[X_j \right] \right).
\end{equation*}
Since the weights are uncorrelated, we have 
\begin{equation*}
    \text{Cov} \left[ g^{(\ell)}, \tilde g^{(\ell)} \right] 
    = \sum_{i = 1}^{H_{\ell - 1}}  \sum_{j = 1}^{H_{\ell - 1}} \mathbb{E}\left[ W_i \right]  \mathbb{E}[ \tilde W_j  ] \left( \mathbb{E}\left[ X_i X_j\right] -  \mathbb{E}\left[ X_i\right]  \mathbb{E}\left[X_j \right] \right).
\end{equation*}
If $\mathbb{E}[W_i] = 0$ for all $i = 1, 
\dots, H_{\ell - 1}$, then $\text{Cov} \left[ g^{(\ell)}, \tilde g^{(\ell)} \right] = 0$.  




\end{proof}

\subsection{Dependence}

We provide an auxiliary lemma that we will further use for the dependence theorem proof. 
\begin{lemma} 
\label{lemma:covariance_sign}
Let $Y$ be a random variable on $\mathbb{R}$ and $\xi_1, \xi_2 : \mathbb{R} \to \mathbb{R}$ be monotonic functions. Then $\mathrm{Cov}(\xi_1(Y), \xi_2(Y)) \ge 0$ if $\xi_1$ and $\xi_2$ have the same monotonicity (are both non-increasing or both non-decreasing), and $\mathrm{Cov}(\xi_1(Y), \xi_2(Y)) \le 0$ otherwise. 
\end{lemma}
\begin{proof}
Let $Y_1$ be an independent copy of $Y$. Let us consider the following expectation: 
\begin{multline*}
    \mathbb{E} \left[\left(\xi_1(Y) - \xi_1(Y_1)\right) \left(\xi_2(Y) - \xi_2(Y_1)\right) \right] = \\
    \mathbb{E} \left[\xi_1(Y) \xi_2(Y) \right] - \mathbb{E} \left[ \xi_1(Y) \xi_2(Y_1)\right] - \mathbb{E} \left[ \xi_1(Y_1) \xi_2(Y) \right]  + \mathbb{E} \left[ \xi_1(Y_1) \xi_2(Y_1)\right]. 
\end{multline*}
The independence of $Y$ and $Y_1$ yields $\mathbb{E} \left[ \xi_1(Y)  \xi_2(Y_1)\right] = \mathbb{E} \left[ \xi_1(Y) \right] \mathbb{E} \left[\xi_2(Y_1)\right]$. Since $Y$ and $Y_1$ are identically distributed, then we get 
\begin{equation*}
     \mathbb{E} \left[\left(\xi_1(Y) - \xi_1(Y_1)\right) \left(\xi_2(Y) - \xi_2(Y_1)\right) \right] = 2 \text{Cov}\left[\xi_1(Y), \xi_2(Y) \right].
\end{equation*}

If $\xi_1$ and $\xi_2$ are both increasing or both decreasing, then, for all $x, y \in \mathbb{R}$, 
\begin{equation*}
    (\xi_1(x) - \xi_1(y))(\xi_2(x) - \xi_2(y)) \ge 0.
\end{equation*}
Otherwise, for all $x, y \in \mathbb{R}$, we have 
\begin{equation*}
    (\xi_1(x) - \xi_1(y))(\xi_2(x) - \xi_2(y)) \le 0.
\end{equation*}
Taking the expectation leads to the conclusion.
\end{proof}

\begin{lemma}
\label{lemma:delta_for_z_positive}
Consider a Bayesian neural network as described in Equation~\eqref{eq:hidden_unit_form} with some activation function. Let elements of weight vector $\bW^{(\ell)}$ follow some zero-center elliptical (possibly different and possibly dependent) distributions, and weight vectors be independent for distinct units of the following layer. 
If~$\ell = 1$, then $\Delta^{(\ell)}(z_1, z_2) = 0$ for all $z_1, z_2$. 
If $\ell \ge 2$, then $\Delta^{(\ell)}(z_1, z_2) \ge 0$ if $z_1 z_2 \ge 0$, and $\Delta^{(\ell)}(z_1, z_2) \le 0$ otherwise. 
\end{lemma}
\begin{proof}
The case where $\ell = 1$ trivially holds as pre-activations are independent for independent weights. 

Consider the case where $\ell \ge 2$.
Let $\bX \in \mathbb{R}^{H_{\ell-1}}$ be outputs of hidden layer $\ell - 1$, $\bW \in \mathbb{R}^{H_{\ell-1}}$ be weights that follow some prior distribution, and $\tilde \bW \in \mathbb{R}^{H_{\ell-1}}$ be an independent copy of $\bW$. Since $g^{(\ell)}$  and $\tilde g^{(\ell)}$ are two distinct units of layer $\ell$, they can be written as $g^{(\ell)} = \bW^\T \bX$ and $\tilde g^{(\ell)} = \tilde \bW^\T \bX$, then 
\begin{align*}
    \mathbb{P}\left(g^{(\ell)} \ge z_1, \tilde g^{(\ell)} \ge z_2 \right) &=  \mathbb{P}\left(\bW^\T \bX \ge z_1, \tilde \bW^\T \bX \ge z_2 \right)  \\
    &= \mathbb{E}\left[ \mathbb{I}\left( \bW^\T \bX \ge z_1, \tilde \bW^\T \bX \ge z_2 \right) \right] \\
    &=\mathbb{E}_X \left[ \mathbb{E}_W\left[ \mathbb{I}\left( \bW^\T \bX \ge z_1, \tilde \bW^\T \bX \ge z_2\right) \right] \Bigl. \Bigr|  \bX \right] \\
    &=\mathbb{E}_X \left[ \mathbb{P}_W\left( \bW^\T \bX \ge z_1, \tilde \bW^\T \bX \ge z_2 \Bigl. \Bigr|  \bX \right) \right].
    \label{lemma:units_dependence:eq_joint_probability_as_product}
\end{align*}
Since the weights $\bW$ and $\tilde \bW$ of different hidden units are independent, pre-activations are independent conditionally on $\bX$. Therefore, we can express the conditional joint probability as a product of conditional probabilities:
\begin{equation*}
    \mathbb{P}_W\left( \bW^\T \bX \ge z_1, \tilde \bW^\T \bX \ge z_2 \Bigl. \Bigr|  \bX \right)
    = \mathbb{P}_W\left( \bW^\T \bX \ge z_1 \Bigl. \Bigr|  \bX \right) \mathbb{P}_W\left( \tilde \bW^\T \bX \ge z_2 \Bigl. \Bigr|  \bX  \right). 
\end{equation*}
Weights $\bW$ and $\tilde \bW$ are identically distributed, so the conditional probabilities differ only by the lower bound values $z_1$ and $z_2$. Therefore, we get 
\begin{equation}
\label{lemma:units_dependence:eq_exp_of_product}
    \mathbb{P}\left(g^{(\ell)} \ge z_1, \tilde g^{(\ell)} \ge z_2 \right) =
    \mathbb{E}_X \left[ \mathbb{P}_W\left( \bW^\T \bX \ge z_1 \Bigl. \Bigr|  \bX \right) \mathbb{P}_W\left( \bW^\T \bX \ge z_2 \Bigl. \Bigr|  \bX \right) \right].
\end{equation}

Now consider the product of probabilities 
\begin{multline}
    \mathbb{P}\left(g^{(\ell)} \ge z_1 \right)  \mathbb{P}\left(\tilde g^{(\ell)} \ge z_2 \right) =  
    \mathbb{P}\left( \bW^\T \bX \ge z_1 \right) \mathbb{P}\left( \tilde \bW^\T \bX \ge z_2 \right)  \\
    =  \mathbb{P}\left( \bW^\T \bX \ge z_1 \right)  \mathbb{P}\left( \bW^\T \bX \ge z_2 \right) \\
    = \mathbb{E}_X \left[ \mathbb{P}_W\left( \bW^\T \bX \ge z_1 \Bigl. \Bigr|  \bX \right)\right] \mathbb{E}_X \left[ \mathbb{P}_W\left( \bW^\T \bX \ge z_2 \Bigl. \Bigr|  \bX \right) \right].
    \label{lemma:units_dependence:eq_product_of_exp}
\end{multline}
Then, by combining Equations~\eqref{lemma:units_dependence:eq_exp_of_product} and \eqref{lemma:units_dependence:eq_product_of_exp}, at the $\ell$-th layer we get 
\begin{equation*}
    \Delta(z_1, z_2) = \text{Cov}\left[\mathbb{P}_W\left( \bW^\T \bX \ge z_1 \Bigl. \Bigr|  \bX \right), \mathbb{P}_W\left( \bW^\T \bX \ge z_2 \Bigl. \Bigr|  \bX \right)\right].
\end{equation*}

Since $\bW$ follows a centered elliptical distribution, then for some positive-definite matrix $\Sigma$ and some scalar function $\psi$ the density function has the form $f(\bw) = \psi(\bw^\T \Sigma^{-1} \bw)$~\citep{cambanis1981theory}. 

Consider the case when $z \not= 0$. From ellipticity we have
\begin{align*}
    \mathbb{P}_W\left( \bW^\T \bX \ge z \Bigl. \Bigr|  \bX \right) &= \int \mathbb{I}\left[ \bw^\T \bX \ge z \right] \psi \left(\bw^\T \Sigma^{-1} \bw \right) \ddr \bw \\
    &= \int \mathbb{I}\left[ \bw^\T \frac{\bX}{\|\bX\|_\Sigma} \ge \frac{z}{\|\bX\|_\Sigma} \right] \psi \left(\bw^\T \Sigma^{-1} \bw \right)  \ddr \bw.
\end{align*}
Introduce the change of variables  $\bv = Q_{\bX}^\T \Sigma^{-\nicefrac12} \bw$ for some rotation (orthogonal) matrix~$Q_{\bX}$ (which satisfies $Q_{\bX}^{-1} = Q_{\bX}^\T$) such that $Q_{\bX}^{-1} \Sigma^{\nicefrac12} \frac{\bX}{\|\bX\|_\Sigma}$ equals the first basis vector $\be_1$. Since $\det(Q_{\bX})=1$ is independent of $\bX$, this shows that
\begin{align*}
    \mathbb{P}_W\left( \bW^\T \bX \ge z \Bigl. \Bigr|  \bX \right)
    &= \int \mathbb{I}\left[ \bv^\T \be_1 \ge \frac{z}{\|\bX\|_\Sigma} \right] \psi \left(\bv^\T \bv \right)  \det(\Sigma^{1/2}) \ddr \bv, 
\end{align*}
thus establishing that function $\bX\mapsto \mathbb{P}_W\left( \bW^\T \bX \ge z \Bigl. \Bigr|  \bX \right)$ is actually a function of~$Y = \|\bX\|_\Sigma$, a one-dimensional random variable, i.e. for $Y > 0$ and for some function~$\xi_z$, $\mathbb{P}_W\left( \bW^\T \bX \ge z \Bigl. \Bigr|  \bX \right) = \xi_z \left( Y \right)$.

Determine $\xi_z(0) = \mathbb{I}[z \le 0]$. Then, 
\begin{equation*}
    \Delta^{(\ell)}(z_1, z_2) = \text{Cov}\left[\xi_{z_1}(Y), \xi_{z_2}(Y)\right].
\end{equation*}

If $z > 0$, then $Y \to \xi_z(Y)$ is non-decreasing  as $\mathbb{I}[z \le 0] = 0 \le \Psi \left( \frac{z}{Y} \right)$. 
Similarly, if $z < 0$, then $Y \to \xi_z(Y)$ is non-increasing. 

Therefore, $\xi_{z_1}$ and $\xi_{z_2}$ have the same monotonicity if $z_1$ and $z_2$ are of the same sign. According to Lemma~\ref{lemma:covariance_sign}, in this case  $\Delta(z_1, z_2) = \mathrm{Cov}(\xi_{z_1}(Y), \xi_{z_2}(Y) )\ge 0$. If $z_1$ and $z_2$ are of different signs, then  $\Delta(z_1, z_2) = \mathrm{Cov}(\xi_{z_1}(Y), \xi_{z_2}(Y) )\le 0$.

If $z = 0$, then $\xi_0(Y) \le 1$ for $Y > 0$ and $\xi_0(0) = 1$. Thus, since at the smallest value the function has the maximum, $\xi_0(Y)$ is non-increasing, and Lemma~\ref{lemma:covariance_sign} can also be applied to the case when $z_1$ or $z_2$ is zero.




\end{proof}

\begin{lemma}
\label{lemma:delta_for_z_equal_zero}
Consider a Bayesian neural network as described in Equation~\eqref{eq:hidden_unit_form} with some activation function $\phi$. Let elements of weight vector $\bW^{(\ell)}$ follow some zero-center elliptical (possibly different and possibly dependent) distributions, and weight vectors be independent for distinct units of the following layer. The activation function satisfies $\mathbb{P}\left(\phi(\bg^{(\ell)}) = 0 \right) = 0$ at layer $\ell$ iff $\Delta^{(\ell)}(0, z) = 0$ for any $z$. 
\end{lemma}
\begin{proof}
With previous notations, we set $z_1 = 0$ and $z_2=z$. Note that we could invert the roles of $z_1$ and $z_2$  without loss of generality. 

Let $\mathbb{P}(\bX = 0) = p$, then $\mathbb{P}(\bX \not= 0) = 1 - p$. Notice that
$\mathbb{P}_W\left( \bW^\T \bX \ge z \Bigl. \Bigr|  \bX = 0\right) = \mathbb{I}[z \le 0]$, and, in particular, if $z = 0$, $\mathbb{P}_W\left( \bW^\T \bX \ge 0 \Bigl. \Bigr|  \bX = 0\right) = 1$. Moreover, $\mathbb{P}_W\left( \bW^\T \bX \ge 0 \Bigl. \Bigr|  \bX \not= 0\right) = \nicefrac12$ due to ellipticity. 

Therefore, for the case when $z = 0$ we have 
\begin{align*}
    \mathbb{E}_X \left[ \mathbb{P}_W\left( \bW^\T \bX \ge 0 \Bigl. \Bigr|  \bX \right)\right] 
    &=\mathbb{E}_X \left[ \mathbb{P}_W\left( \bW^\T \bX \ge 0 \Bigl. \Bigr|  \bX = 0 \right)\right] p \\
    &+ \mathbb{E}_X \left[ \mathbb{P}_W\left( \bW^\T \bX \ge 0 \Bigl. \Bigr|  \bX \not= 0 \right) \right] (1 - p) \\
    &= p + \frac{1 - p}2  = \frac{p + 1}2,  
\end{align*}
\begin{align*}
    \mathbb{E}_X \left[ \mathbb{P}_W^2 \left( \bW^\T \bX \ge 0 \Bigl. \Bigr|  \bX \right)\right] 
    &=  \mathbb{E}_X \left[ \mathbb{P}_W^2\left( \bW^\T \bX \ge 0 \Bigl. \Bigr|  \bX = 0 \right)\right] p \\
    &+ \mathbb{E}_X \left[ \mathbb{P}_W^2\left( \bW^\T \bX \ge 0 \Bigl. \Bigr|  \bX \not= 0 \right)\right] (1 - p) \\
    &=  p + \frac{1 - p}4 = \frac{3p + 1}4.
\end{align*}
Thus, we get
\begin{equation*}
    \Delta^{(\ell)}(0, 0) = \frac{3p + 1}4 - \frac{(p + 1)^2}4 = \frac{p (1 - p)}4 \ge 0.
\end{equation*}
We see that $ \Delta^{(\ell)}(0, 0) = 0$ iff $p = \mathbb{P}(\bX = 0) = 0$ or $1 - p = \mathbb{P}(\bX \not= 0) = 0$.

Now let us consider more general case, where $z \not= 0$: 
\begin{align*}
    \mathbb{E}_X \left[ \mathbb{P}_W\left( \bW^\T \bX \ge z \Bigl. \Bigr|  \bX \right)\right] 
    &=\mathbb{E}_X \left[ \mathbb{P}_W\left( \bW^\T \bX \ge z \Bigl. \Bigr|  \bX = 0 \right)\right] p \\
    &+ \mathbb{E}_X \left[ \mathbb{P}_W\left( \bW^\T \bX \ge z \Bigl. \Bigr|  \bX \not= 0 \right) \right] (1 - p) \\
    &= p\, \mathbb{I}[z \le 0] + (1 - p) \mathbb{E}_X \left[ \mathbb{P}_W\left( \bW^\T \bX \ge z \Bigl. \Bigr|  \bX \not= 0 \right)\right],  
\end{align*}
\begin{align*}
    \mathbb{E}_X &\left[ \mathbb{P}_W\left( \bW^\T \bX \ge 0 \Bigl. \Bigr|  \bX \right)\mathbb{P}_W\left( \bW^\T \bX \ge z \Bigl. \Bigr|  \bX \right)\right] \\
    &=  \mathbb{E}_X \left[ \mathbb{P}_W\left( \bW^\T \bX \ge 0 \Bigl. \Bigr|  \bX = 0 \right)\mathbb{P}_W\left( \bW^\T \bX \ge z \Bigl. \Bigr|  \bX = 0 \right)\right] p \\
    &+ \mathbb{E}_X \left[ \mathbb{P}_W\left( \bW^\T \bX \ge 0 \Bigl. \Bigr|  \bX \not= 0 \right)\mathbb{P}_W\left( \bW^\T \bX \ge z \Bigl. \Bigr|  \bX \not= 0 \right)\right] (1 - p) \\
    &=  p \, \mathbb{I}[z \le 0] + \frac{1 - p}2 \mathbb{E}_X \left[ \mathbb{P}_W\left( \bW^\T \bX \ge z \Bigl. \Bigr|  \bX \not= 0 \right)\right].
\end{align*}
Further, 
\begin{align*}
    \Delta^{(\ell)}(0, z) &= p \, \mathbb{I}[z \le 0] + \frac{1 - p}2 \mathbb{E}_X \left[ \mathbb{P}_W\left( \bW^\T \bX \ge z \Bigl. \Bigr|  \bX \not= 0 \right)\right] \\
    &- \frac{p + 1}2 \left(p\, \mathbb{I}[z \le 0] + (1 - p) \mathbb{E}_X \left[ \mathbb{P}_W\left( \bW^\T \bX \ge z \Bigl. \Bigr|  \bX \not= 0 \right)\right] \right) \\
    &= \frac{p(1 - p)}2 \, \mathbb{I}[z \le 0]  - \frac{p(1 - p)}2 \,\mathbb{E}_X \left[ \mathbb{P}_W\left( \bW^\T \bX \ge z \Bigl. \Bigr|  \bX \not= 0 \right)\right].
\end{align*}
If $z > 0$, then $\Delta^{(\ell)}(0, z) = - \frac{p(1 - p)}2 \,\mathbb{E}_X \left[ \mathbb{P}_W\left( \bW^\T \bX \ge z \Bigl. \Bigr|  \bX \not= 0 \right)\right] \le 0$. 

If $z < 0$, then $\Delta^{(\ell)}(0, z) = \frac{p(1 - p)}2 \left(1 - \mathbb{E}_X \left[ \mathbb{P}_W\left( \bW^\T \bX \ge z \Bigl. \Bigr|  \bX \not= 0 \right)\right] \right) \ge 0$, as $\mathbb{E}_X \left[ \mathbb{P}_W\left( \bW^\T \bX \ge z \Bigl. \Bigr|  \bX\right)\right] \le 1$. 

Notice that $\Delta^{(\ell)}(0, z) = 0 $ iff $p = 0$ or $1 - p = 0$ for any $z$. 

The case when $p = 1$ means that $\bX = \varphi\left(\bg^{(\ell)}\right) = 0$ for any $\bg^{(\ell)}$. It cannot be the case for an activation function, thus, we get the statement of the lemma. 
\end{proof}

\subsection{Corollaries}

Corollary~\ref{corollary:tau_rho} requires a generalization of Theorem~\ref{theorem:hidden_units_dependence} to sums and differences of pre-activations. Let $\Delta_2^+$ and $\Delta_2^-$ be defined as 
\begin{align}
    \Delta_2^+(z_1, z_2) &:= \mathbb{P}(g_1^{(\ell)} + g_2^{(\ell)} \geq z_1, \tilde g_1^{(\ell)} + \tilde g_2^{(\ell)}\geq z_2) - \mathbb{P}(g_1^{(\ell)} + g_2^{(\ell)}\geq z_1)\mathbb{P}(\tilde g_2^{(\ell)} + \tilde g_2^{(\ell)}\geq z_2), \\
    \Delta_2^-(z_1, z_2) &:= \mathbb{P}(g_1^{(\ell)} - g_2^{(\ell)} \geq z_1, \tilde g_1^{(\ell)} - \tilde g_2^{(\ell)}\geq z_2) - \mathbb{P}(g_1^{(\ell)} - g_2^{(\ell)}\geq z_1)\mathbb{P}(\tilde g_2^{(\ell)} - \tilde g_2^{(\ell)}\geq z_2),
\end{align}
where $g_1^{(\ell)}$, $g_2^{(\ell)}$ are independent copies of hidden unit $g^{(\ell)}$, and  $\tilde g_1^{(\ell)}$, $\tilde g_2^{(\ell)}$ are independent copies of hidden unit $\tilde g^{(\ell)}$.
\begin{theorem}
\label{theorem:independent_copies_dependence}
Under the assumptions of Theorem~\ref{theorem:hidden_units_dependence}, the same result holds for $\Delta_2^+$ and $\Delta_2^-$.
\end{theorem}
\begin{proof}
We say $g^{(\ell)} = \bW^\T \bX$, where $\bX \in \mathbb{R}^{H_{\ell-1}}$ be outputs of hidden layer $\ell - 1$, $\bW \in \mathbb{R}^{H_{\ell-1}}$ be weights that follow some prior distribution, independent of $\bX$. Similarly,   $\tilde g^{(\ell)} = \tilde \bW^\T \bX$, where $\tilde \bW \in \mathbb{R}^{H_{\ell-1}}$ be independent copy of $\bW$. 
We can express the joint probability in $\Delta_2^+$ as
\begin{equation*}
    \mathbb{P}\left(g_1^{(\ell)} + g_2^{(\ell)} \geq z_1, \tilde g_1^{(\ell)} + \tilde g_2^{(\ell)}\geq z_2 \right) = \mathbb{P}\left(\bW_1^\T \bX_1 + \bW_2^\T \bX_2\ge z_1, \tilde \bW_1^\T \bX_1 +  \tilde \bW_2^\T \bX_2 \ge z_2 \right).
\end{equation*}
Following the proof of Theorem~\ref{theorem:hidden_units_dependence}, we have 
\begin{multline*}
    \mathbb{P}\left(g_1^{(\ell)} + g_2^{(\ell)} \geq z_1, \tilde g_1^{(\ell)} + \tilde g_2^{(\ell)}\geq z_2 \right) \\ = \mathbb{P}\left(\bW_1^\T \bx_1 + \bW_2^\T \bx_2\ge z_1, \tilde \bW_1^\T \bx_1 +  \tilde \bW_2^\T \bx_2 \ge z_2 \Bigl. \Bigr|  \bX_1 = \bx_1, \bX_2 = \bx_2 \right).
\end{multline*}
Let us denote $\bW_0 = [\bW_1, \bW_2] \in \mathbb{R}^{2H_{\ell-1}}$, $\tilde \bW_0 = [\tilde \bW_1, \tilde \bW_2] \in \mathbb{R}^{2H_{\ell-1}}$, $\bX_0 = [\bX_1, \bX_2] \in \mathbb{R}^{2H_{\ell-1}}$, and $\bx_0 = [\bx_1, \bx_2] \in \mathbb{R}^{2H_{\ell-1}}$. We obtain $\bW_0$ and $\tilde \bW_0$ are vectors of elliptical distributions independent of $\bX_0$. 
Now, we can rewrite 
\begin{align*}
    \mathbb{P}\left(g_1^{(\ell)} + g_2^{(\ell)} \geq z_1, \tilde g_1^{(\ell)} + \tilde g_2^{(\ell)}\geq z_2 \right) &= \mathbb{P}\left(\bW_0^\T \bx_0 \ge z_1, \tilde \bW_0^\T \bx_0 \ge z_2 \Bigl. \Bigr|  \bX_0 = \bx_0 \right) \\
    &= \mathbb{P}\left(\bW_0^\T \bx_0 \ge z_1\Bigl. \Bigr|  \bX_0 = \bx_0 \right) \mathbb{P}\left(\tilde \bW_0^\T \bx_0 \ge z_2 \Bigl. \Bigr|  \bX_0 = \bx_0 \right).
\end{align*}
The same way we get an equation for a product of probabilities 
\begin{equation*}
    \mathbb{P}\left(g_1^{(\ell)} + g_2^{(\ell)} \ge z_1 \right)  \mathbb{P}\left(\tilde g_1^{(\ell)} + \tilde g_2^{(\ell)} \ge z_2 \right) =  
    \mathbb{P}\left( \bW_0^\T \bx_0 \ge z_1 \right) \mathbb{P}\left( \tilde \bW_0^\T \bx_0 \ge z_2 \right).
\end{equation*}
The rest of the proof is exactly the same as in Theorem~\ref{theorem:hidden_units_dependence}.

Notice that if $\bW$ is elliptical, then $-\bW$ is elliptical. Then, for the case of $\Delta_2^-$, we denote $\bW_0 = [\bW_1, - \bW_2]$ and $\tilde \bW_0 = [\tilde \bW_1, - \tilde \bW_2]$, which are also elliptical vectors independent of $\bX_0$. Similarly as for $\Delta_2^+$, we obtain the statement for $\Delta_2^-$.
\end{proof}

\begin{manualcorollary}{\ref{corollary:tau_rho}}
\corollaryTauRho
\end{manualcorollary}

\begin{proof}
Consider random variables $(X, Y)$ with some joint distribution. Let $(X_1, Y_1)$ and $(X_2, Y_2)$ be independent and identically distributed random copies of $(X, Y)$. From~\citet{nelsen2007introduction}, Kendall's tau $\tau$ can be expressed as 
\begin{equation*}
    \tau = \tau_{X, Y} = \mathbb{P}\left[ (X_1 - X_2)(Y_1 - Y_2) > 0\right] - \mathbb{P}\left[ (X_1 - X_2)(Y_1 - Y_2) < 0\right]. 
\end{equation*}
Let $X$ and $Y$ be different hidden units from Bayesian neural networks satisfying the assumptions in the statement.

Notice that $\mathbb{P}\left[ (X_1 - X_2)(Y_1 - Y_2)Y > 0 \right] = \mathbb{P}\left[ X_1 - X_2 > 0, Y_1 - Y_2 > 0 \right] + \mathbb{P}\left[ X_1 - X_2 < 0, Y_1 - Y_2 < 0 \right]$. From  Theorem~\ref{theorem:independent_copies_dependence}, we have $\Delta(0, 0)=0$, so $\mathbb{P}\left[ X_1 - X_2  > 0, Y_1 - Y_2 > 0 \right] = \mathbb{P}\left[ X_1 - X_2  > 0\right]\mathbb{P}\left[ Y_1 - Y_2 > 0 \right]$. Since $X_1$ and $X_2$ are independent copies of $X$, $\mathbb{P}\left[ X_1 > X_2 \right] = 1/2$
Similarly, combining Theorem~\ref{theorem:independent_copies_dependence} with  Remark~\ref{remark:delta_with_upper_bounds}, $\mathbb{P}\left[ X_1 - X_2  < 0, Y_1 - Y_2 < 0 \right] = \mathbb{P}\left[ X_1 - X_2 < 0\right]\mathbb{P}\left[ Y_1 - Y_2 < 0 \right]$ and $\mathbb{P}\left[ X_1 < X_2 \right] = 1/2$. Therefore, $\tau = 0$.

Spearman's rho $\rho$ is defined as 
\begin{equation*}
     \rho = \rho_{X, Y} = 3\left(\mathbb{P}\left[ (X_1 - X_2)(Y_1 - Y_3) > 0\right] - \mathbb{P}\left[ (X_1 - X_2)(Y_1 - Y_3) < 0\right]\right), 
\end{equation*}
where $(X_1, Y_1)$, $(X_2, Y_2)$ and $(X_3, Y_3)$ are independent and identically distributed random copies of~$(X, Y)$~\citep{nelsen2007introduction}. The proof for $\rho$ is identical.
\end{proof}

\end{document}